\documentclass[11pt,a4paper]{scrartcl}
\usepackage{graphicx}
\usepackage{amsfonts}
\usepackage{amsthm}
\usepackage{amssymb}
\usepackage{amstext}
\usepackage{amsmath}
\usepackage{amsthm}
\usepackage{pstricks}
\usepackage{pspicture}
\usepackage{amscd}
\usepackage{appendix}
\usepackage{subcaption}
\usepackage{tipa}
\usepackage{pifont}
\usepackage{natbib}

\newcommand{\cmark}{\ding{51}}%
\newcommand{\xmark}{\ding{55}}%

 \newtheorem{proposition}{Proposition}
  \newtheorem{corollary}{Corollary}
  \newtheorem{lemma}{Lemma}
 \newtheorem{definition}{Definition}
  
  \newtheorem{remark}{Remark}

  \theoremstyle{definition}
  \newtheorem{example}{Example}

\begin{document}
\title{\bf Boundary properties of the inconsistency of pairwise comparisons in group decisions}
\author{
{\bf Matteo Brunelli}
\\
{\normalsize  Department of Mathematics and Systems Analysis} \\
{\normalsize Aalto University}, {\normalsize P.O. box 11100,
00076 Aalto, Finland}
\\ {\normalsize e-mail:
\texttt{matteo.brunelli@aalto.fi}}
\vspace{0.3cm}\\
{\bf Michele Fedrizzi}
\\
{\normalsize Department of Industrial Engineering} \\
{\normalsize University of Trento}, {\normalsize Via Sommarive 9, I-38123 Trento, Italy}
\\ {\normalsize e-mail:
\texttt{michele.fedrizzi@unitn.it}}
}
\date{\today}

\maketitle \thispagestyle{empty}


\begin{center}
{\bf Abstract}
\end{center}

{\small \noindent This paper proposes an analysis of the effects of consensus and preference aggregation on the consistency of pairwise comparisons. We define some boundary properties for the inconsistency of group preferences and investigate their relation with different inconsistency indices. Some results are presented on more general dependencies between properties of inconsistency indices and the satisfaction of boundary properties. In the end, given three boundary properties and nine indices among the most relevant ones, we will be able to present a complete analysis of what indices satisfy what properties and offer a reflection on the interpretation of the inconsistency of group preferences.}

 \vspace{0.3cm}
 \vspace{0.3cm}


\section{Introduction}

In a wide range of decision making problems, it occurs that an expert, or a group of experts, is asked to rate some alternatives. Selecting the best alternative is trivial when the number of considered alternatives is very small, but complexity arises as the number of alternatives, and criteria with respect to which alternatives are judged, grows. Techniques based on pairwise comparisons allow the expert to discriminate between two alternatives at a time, thus decomposing the problem into more simple and easily tractable sub-problems. \\The Analytic Hierarchy Process (AHP) by \citet{Saaty1977,Saaty1980} is probably the best known among all the methods using pairwise comparisons. In a recent survey by \citet{IshizakaLabib2011} on the latest developments of the AHP, consistency of preferences and group decisions have been considered hot topics, and the possibility of estimating inconsistency been regarded as a valuable asset for techniques adopting pairwise comparison matrices.

Consistency has been widely regarded as a desirable, yet hardly ever achievable, property of preferences in decision making problems. Following the thesis of \citet{Irwin1958}, which links the concepts of preference and discrimination, being consistent in expressing preferences means being rational in discriminating between alternatives. Although one might argue that consistency does not necessarily imply expertise of the decision maker (consistent preferences could possibly be obtained randomly), it is undebatable that a good expert should always be able to state his preferences in a non-contradictory way. Hence, although consistency alone does not guarantee the expertise of a decision maker, the existence of inconsistencies should be symptomatic of the decision maker's scarce preparation or lack of knowledge of the problem. Going back in time, even in the fundamental contribution by \citet{Savage1972} consistency of preferences was regarded as a \textit{desideratum}. More recently, \citet{Gass2005} recalled that also \citet{LuceRaiffa1957} and \citet{Fishburn1991} regarded transitivity of preferences, and consequently their consistency, as an auspicable, but not necessary, condition for the preferences of a decision maker.


Note that the use of the concept of consistency has not been limited to the quantification of inconsistency. For example, just to cite the most recent results, it has been used to improve consistency of preferences in the framework of a model based on Hadamard product between matrices \citep{KouErguShang2014}, to detect the most inconsistent comparisons \citep{ErguEtAl2011}, and to derive the priority vector by means of a geometric similarity measure \citep{KouLin2014}. On the other hand, also the concept of group decisions with pairwise comparisons has been studied; for example \citet{AltuzarraEtAl2010} proposed a statistical model for consensus in the AHP including assumptions on group consistency, and \citet{BernasconiEtAl2014} studied, from the empirical point of view, different aggregation methods for preferences.

Consistency has been a widely studied research topic in decision sciences, but, in spite of the growing effort in studying group decisions, most of the research on consistency of preference relations has focused on the reliable assessment of the degree of inconsistency of single pairwise comparison matrices. Only few, and more recent, studies \citep{EscobarEtAl2004,Groselj,LinEtAl2008,LiuEtAl2012} have tried to extend the issue of inconsistency quantification to the case with multiple decision makers by examining single measures of inconsistency but never proposing a more general reasoning on the matter.
Furthermore, very few studies have investigated the connection between consensus and consistency. In a qualitative study, \citet{WeissShanteau2004} highlighted how consensus alone does not necessarily lead to better decisions and, instead, emphasized the fundamental role of the expertise of decision makers, which they called consistency.

In this paper we will provide further results on the connection between group decisions and inconsistency, in particular on how the former affects the latter. More specifically we shall define some general boundary properties for the inconsistency of a group of decision makers and see whether different inconsistency indices satisfy them or not. In doing so, we shall be able to derive and use some more general results starting from some axiomatic properties of inconsistency indices \citep{BrunelliFedrizziJORS}.


The paper is outlined as follows.  In Section \ref{sec:pairwise} we
recall the definitions of pairwise comparison matrices and
inconsistency, and we summarize the axioms which were proposed to
characterize inconsistency indices. In Section \ref{sec:boundaries}
we define the boundary properties and, within the same section, in
Subsections \ref{sub:lower} and \ref{sub:upper} we study the
satisfaction of lower and upper boundary properties. In Section
\ref{sec:discussion} we discuss the implications of the results, and
in Section \ref{sec:conclusions} we draw the conclusions.

\section{Pairwise comparison matrices and inconsistency indices}
\label{sec:pairwise}
The intensity of pairwise preferences of a decision maker can be represented on bipolar scales. The approach proposed by \citet{Fishburn1991} based on skew symmetric additive preferences considers the opinions of a decision maker to be expressed on the real line with the value $0$ representing indifference between two alternatives. Conversely, \citet{LuceSuppes1965} and part of the fuzzy sets community \citep{DeBaetsEtAl2006} studied judgments expressed on the scale $]0,1[$ with the indifference represented by the value $0.5$. Hereafter, due to its popularity, we shall instead consider the approach offered by \citet{Saaty1977,Saaty1980}, where pairwise judgments are expressed as entries of positive reciprocal matrices, often called pairwise comparison matrices. Even so, our conclusions should not lose in generality as it was proven that all the above mentioned approaches are group isomorphic to each other \citep{CavalloD'Apuzzo2009}.
Given a set of $n$ alternatives, a \emph{pairwise comparison matrix} is a positive square matrix $\mathbf{A}=(a_{ij})_{n \times n}$ such that $a_{ij}a_{ji}=1 \, \forall i,j$, where $a_{ij}$ is the subjective assessment of the relative importance of the $i$-th alternative with respect to the $j$-th. For instance, $a_{ij}=2$ means that, for the decision maker, the $i$-th alternative is two times better than the $j$-th.
A pairwise comparison matrix is \emph{consistent} if and only if
\begin{equation}
\label{eq:consistency}
a_{ik}=a_{ij}a_{jk}~~~\forall i,j,k.
\end{equation}
Furthermore, if and only if a pairwise comparison matrix $\mathbf{A}$ is consistent, there exists a vector $\mathbf{w}=(w_{1},\ldots,w_{n})$ such that
\[
a_{ij}=\frac{w_{i}}{w_{j}}~~~\forall i,j.
\]
A possible way of finding vector $\mathbf{w}$ from a consistent pairwise comparison matrix is the geometric mean method
\[
w_{i}=\left( \prod_{j=1}^{n}a_{ij} \right)^{\frac{1}{n}}.
\]
The same method is also commonly used to estimate reliable vectors from inconsistent pairwise comparison matrices.\\
For notational convenience, we define the set of all pairwise comparison matrices as
\[
\mathcal{A}= \left\{ \mathbf{A}=(a_{ij})_{n \times n} | a_{ij}>0, a_{ij}a_{ji}=1 ~\forall i,j, ~ n>2 \right\}.
\]
Similarly, the set of all \emph{consistent} pairwise comparison matrices $\mathcal{A}^{*} \subset  \mathcal{A}$ is defined as
\[
\mathcal{A}^{*}= \{ \mathbf{A}=(a_{ij})_{n \times n} | \mathbf{A} \in \mathcal{A}, a_{ik} = a_{ij}a_{jk} ~ \forall i,j,k \}
\]

A very precise definition of consistency was given in
(\ref{eq:consistency}), but in the literature there is not a
meeting of minds on how inconsistency should be quantified. In fact,
inconsistency is generally regarded as a lack of consistency, i.e. a
deviation from (\ref{eq:consistency}), but there is not a unique
formula to quantify it. To overcome this problem, inconsistency
indices have been introduced. Inconsistency indices are functions
\begin{equation}
\label{eq:inconsistency_index}
I:\mathcal{A} \rightarrow \mathbb{R} \, ,
\end{equation}
where the value $I(\mathbf{A})$ is an estimation of the degree of
inconsistency of the  pairwise comparison matrix $\mathbf{A}$. It is
important to note that each inconsistency index is in fact a
different \emph{definition} of inconsistency measuring. Establishing
if a function returns a reasonable estimation of
inconsistency---i.e. if a function is a good definition of
inconsistency---is a crucial point indeed. In fact, there exist
infinitely many functions (\ref{eq:inconsistency_index}) whose
behavior is obviously meaningless when it comes to estimate the
degree of inconsistency of a pairwise comparison matrix. This is the
reason which motivated the introduction of some minimal reasonable
properties that any inconsistency index should satisfy
\citep{BrunelliFedrizziJORS}. The five axiomatic properties are
summarized and justified in the following and will later be used to
derive some results. We refer to the original paper
\citep{BrunelliFedrizziJORS} for more detailed descriptions and
comments.
\begin{description}
    \item[A1:] There exists a unique $\nu \in \mathbb{R}$ representing the situation of full consistency, i.e.
    \[
     \label{P1}
\exists ! \nu \in \mathbb{R} \text{ such that } I(\mathbf{A})= \nu \Leftrightarrow \mathbf{A} \in \mathcal{A}^{*}.
    \]
    Hence, every inconsistency index should at least be able to distinguish between fully consistent and inconsistent matrices.
    \item[A2:] Changing the order of the alternatives does not affect the inconsistency of preferences. That is,
    \begin{equation*}
\label{eq:permutation}
I(\mathbf{P}\mathbf{A}\mathbf{P}^{T})=I(\mathbf{A})
\end{equation*}
for any permutation matrix $\mathbf{P}$. Thus, inconsistency remains
unchanged when the names of alternatives are exchanged.

\item[A3:] If preferences of a
matrix $\mathbf{A}$ are properly intensified obtaining a new matrix
which we denote by $\mathbf{A}(b)$, then the inconsistency of
$\mathbf{A}(b)$ cannot be smaller than the inconsistency of
$\mathbf{A}$. In fact, if all the expressed preferences indicate
indifference between alternatives, it is $a_{ij}=1 \forall i, j$,
and $\mathbf{A}$ is consistent. Going farther from this uniformity
means having sharper and stronger judgments, and this should not
make their possible inconsistency less evident. In other words,
intensifying the preferences (pushing them away from indifference)
should not de-emphasize the characteristics of these preferences and
their possible contradictions. More formally, it was proved by \citet{Saaty1977} that the exponential is the only non-trivial function preserving reciprocity and consistency (when $\mathbf{A}$ is consistent). Thus, the intensification of preferences is obtained defining $\mathbf{A}(b)=\left( a_{ij}^{b} \right)_{n
\times n}$ with $b>1$. Then, the property is as follows
\[
I(\mathbf{A}(b)) \geq I(\mathbf{A}) ~~~~\forall \mathbf{A} \in \mathcal{A}, ~~b > 1 .
\]

\item[A4:] Given  a \emph{consistent}
pairwise comparison matrix and considering a single arbitrary
comparison between two alternatives, then as we push its value far
from its original one, we clearly increase the distance from
consistency. Axiom 4 requires that the inconsistency of the matrix
should not decrease. More formally, given a consistent matrix
$\mathbf{A} \in \mathcal{A}^{*}$, and considering any arbitrary
non-diagonal element $a_{pq}$ (and its reciprocal $a_{qp}$) such
that $a_{pq} \neq 1$, let $\mathbf{A}(\delta)$ be the inconsistent
matrix obtained from \textbf{A} by replacing the entry $a_{pq}$ with
$a_{pq}^{\delta}$, where $\delta \neq 1$. Necessarily, $a_{qp}$ must
be replaced by $a_{qp}^{\delta}$ in order to preserve reciprocity.
Let $\mathbf{A}(\delta')$ be the inconsistent matrix obtained from
\textbf{A} by replacing entries $a_{pq}$ and $a_{qp}$ with
$a_{pq}^{\delta'}$ and $a_{qp}^{\delta'}$ respectively. The property
can then be formulated for all $\mathbf{A} \in \mathcal{A}^{*}$ as
\begin{equation}
\label{monotonicity}
\begin{split}
 \delta' > \delta > 1 & \Rightarrow I(\mathbf{A}(\delta')) \geq I(\mathbf{A}(\delta)) \\
 \delta' < \delta < 1 & \Rightarrow I(\mathbf{A}(\delta')) \geq I(\mathbf{A}(\delta)) .
\end{split}
\end{equation}

\item[A5:] Function $I$ is continuous with respect to the entries of $\mathbf{A}$.
This is required, as infinitesimally small variations of the
preferences  should cause infinitesimally small changes of the value
of the inconsistency.
\end{description}

We remark that all the indices described in this paper measure the
\emph{inconsistency} of the preferences. Nevertheless, in order to
avoid multiple labels, we maintain the original names for the most
popular indices, e.g. `Consistency Index' and `Geometric Consistency
Index'.

\section{The effects of consensus on the consistency of preferences}
\label{sec:boundaries}

In the typical group decision framework, there are $m~(m \geq 2)$ decision makers, each associated with a set of preferences. Hence, there exists $m$ pairwise comparison matrices in the form $\mathbf{A}_{1}=(a_{ij}^{(1)})_{n \times n},\ldots,\mathbf{A}_{m}=(a_{ij}^{(m)})_{n \times n}$. 

In our opinion, it would be interesting to study if, and how, the inconsistency of the pairwise comparisons changes under some circumstances. When studying the connection between consensus and inconsistency, some natural questions could be the following:
\begin{itemize}
    \item How does the inconsistency of the preferences of the single decision makers react when they negotiate and their preferences converge to a consensual solution?
    \item Is the group inconsistency of the aggregated preferences a weighted mean of the inconsistencies of the original preference relations or \emph{systematically} higher/lower?
\end{itemize}

In some recent papers, the problem of computing an upper bound for group inconsistency has been addressed only for a couple of inconsistency indices. In particular, \citet{Xu2000} allegedly proved that Saaty's Consistency Index $CI$ of a combination of pairwise comparison matrices cannot be greater that the maximum $CI$ of the single pairwise comparison matrices. However, \citet{LinEtAl2008} showed that the proof was not satisfactory and that Xu's result was a conjecture. Finally, \citet{LiuEtAl2012} provided a proof showing that Xu's conjecture was, in effect, true. \citet{Groselj} noted that the whole controversy was based on the unawareness that a more general problem had already been solved before by \citet{ElsnerEtAl1988}. Some studies have been proposed for another inconsistency index; \citet{EscobarEtAl2004} investigated the upper boundary of the Geometric Consistency Index, $GCI$.

The question of how the pairwise comparison matrices of different decision makers should be aggregated was answered when \citet{AczelSaaty1983} proved that the weighted geometric mean is the only reasonable function to do so. Hence, when it comes to synthesize $m$ pairwise comparison matrices into a single one $\mathbf{A}^{\star}=(a_{ij}^{\star})_{n \times n}$, then its entries should be obtained as the weighted geometric means of the corresponding entries of the decision makers' pairwise comparison matrices, e.g.
\begin{equation}
\label{eq:aggregation}
a_{ij}^{\star}= \prod_{h=1}^{m}  {a_{ij}^{(h)}}^{\lambda_{h}}
\end{equation}
where $\boldsymbol{\lambda}=(\lambda_{1},\ldots,\lambda_{m})$ such that $\lambda_{h} \geq 0~\forall h$ and $\sum_{h=1}^{m}\lambda_{h}=1$ is the weight vector of relative importances of the decision makers. We use $\mathcal{L}_{m}$ to denote the set of all the weight vectors with $m$ components. That is,
\begin{equation}
\label{eq:L_m}
\mathcal{L}_{m}=\left\{ (\lambda_{1},\ldots,\lambda_{m}) \bigg| \lambda_{h} \geq 0~\forall h,~~\sum_{h=1}^{m}\lambda_{h}=1 \right\}.
\end{equation}
Let us then define some properties which will play a pivotal role in the rest of the paper.
\begin{definition}[Boundary properties]
\label{def:boundary}
Let $\mathbf{A}^{\star}$ be the aggregated pairwise comparison matrix as in (\ref{eq:aggregation}). A function $I:\mathcal{A} \rightarrow \mathbb{R}$ is \emph{lower bounded} (w.r.t the geometric mean) if:
\begin{equation}
\label{eq:lower_bounded}
I \left( \mathbf{A}^{\star} \right) \geq \min \left\{ I \left( \mathbf{A}_{1} \right),\ldots,I \left( \mathbf{A}_{m} \right) \right\}~~\forall \mathbf{A}_{1},\ldots,\mathbf{A}_{m} \in \mathcal{A},~~\boldsymbol{\lambda} \in \mathcal{L}_{m}
\end{equation}
and \emph{upper bounded} (w.r.t. the geometric mean) if
\begin{equation}
\label{eq:upper_bounded}
I \left( \mathbf{A}^{\star} \right) \leq \max \left\{ I \left( \mathbf{A}_{1} \right),\ldots,I \left( \mathbf{A}_{m} \right) \right\} ~~\forall \mathbf{A}_{1},\ldots,\mathbf{A}_{m} \in \mathcal{A},~~\boldsymbol{\lambda} \in \mathcal{L}_{m}
\end{equation}
A function $I:\mathcal{A} \rightarrow \mathbb{R}$ is \emph{strongly upper bounded} (w.r.t the geometric mean) if and only if
\begin{equation}
\label{eq:strongly_upper_bounded}
I(\mathbf{A}^{\star}) \leq  \sum_{h=1}^{m} \lambda_{h} I(\mathbf{A}_{h}) ~~~~\forall \mathbf{A}_{1},\ldots,\mathbf{A}_{m} \in \mathcal{A},~~\boldsymbol{\lambda} \in \mathcal{L}_{m}
\end{equation}
\end{definition}
Note that, even if not specified, the previous definition is for all $n,m\geq 2$. Furthermore, Figure \ref{fig:boundary_properties} provides a graphical interpretation of these properties in the case of two pairwise comparison matrices.

\begin{figure}[h]
         \centering
         \begin{subfigure}[b]{0.34\textwidth}
                 \includegraphics[width=\textwidth]{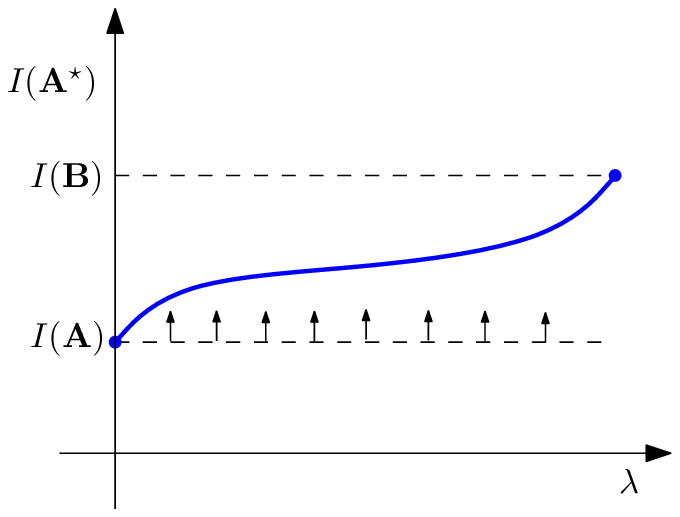}
                 \caption{Lower boundary property}
                 \label{fig:lower}
         \end{subfigure}%
				\hspace{1cm}
         ~ 
         \begin{subfigure}[b]{0.34\textwidth}
                 \includegraphics[width=\textwidth]{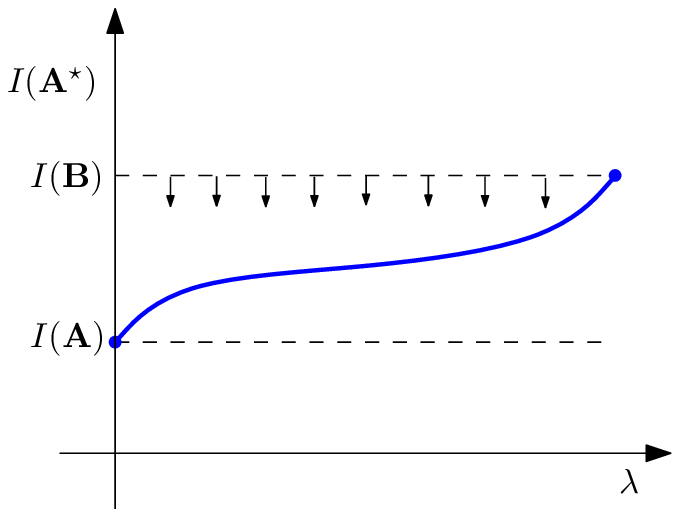}
                 \caption{Upper boundary property}
                 \label{fig:upper}
         \end{subfigure}\\
         ~ 
         \begin{subfigure}[b]{0.34\textwidth}
                 \includegraphics[width=\textwidth]{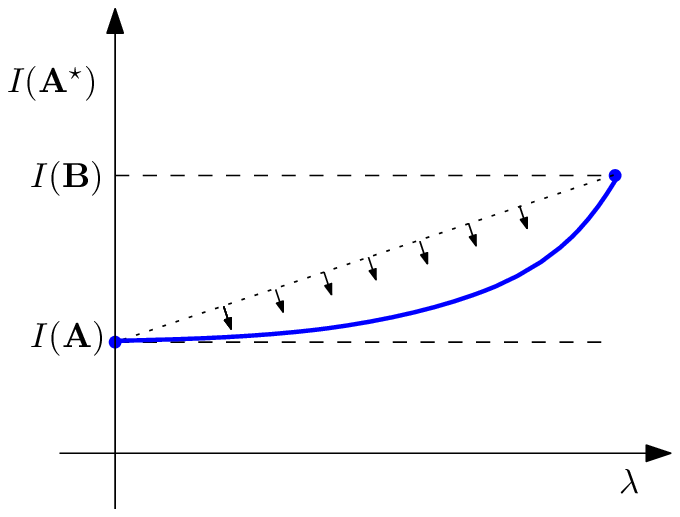}
                 \caption{Strong upper boundary property}
                 \label{fig:strong_upper}
         \end{subfigure}
         \caption{Interpretations of boundary properties in the case of two decision makers with pairwise comparison matrices $\mathbf{A}$ and $\mathbf{B}$. We use the convention $\lambda_{1}=\lambda$ and $\lambda_{2}=1-\lambda$.}\label{fig:boundary_properties}
\end{figure}

\begin{remark}
Given a set of $m$ matrices $\mathbf{A}_{1},\ldots,\mathbf{A}_{m}$, if a function $I$ is strongly upper bounded, then it is also convex w.r.t. the variables $(\lambda_{1},\ldots,\lambda_{m})\in \mathcal{L}_{m}$. In fact, if a function $I$ is \emph{not} convex, then, there exists a subset $\{\mathbf{A}_{1},\ldots,\mathbf{A}_{s} \} \subset \mathcal{A}$ and a vector $(\lambda_{1}^{'},\ldots,\lambda_{s}^{'})\in \mathcal{L}_{s}$ such that $I(\mathbf{A}^{\star})>\sum_{i=1}^{s} \lambda_{i} I(\mathbf{A}_{i})$, which contradicts the definition of strongly upper bounded inconsistency index.
\end{remark}
We believe that it is important to study the phenomenon of inconsistency in the wider context of group decisions and raise the level of the discussion from single indices to their properties A1--A5. In fact, at this point some questions may arise concerning the connection between the axiomatic properties A1--A5, here recalled in Section \ref{sec:pairwise}, and the boundary properties of Definition \ref{def:boundary}. Namely, before trying to prove whether or not each single inconsistency index satisfies the boundary properties, we shall use the axiomatic framework A1--A5 to derive more general results. For instance, one type of more general result would be that, if some axiomatic properties hold for one index $I$, then it is guaranteed that also some of the boundary conditions in Definition \ref{def:boundary} must hold/not hold for $I$.

\subsection{Lower boundary of inconsistency indices}
\label{sub:lower}

If an inconsistency index is lower bounded, then it is impossible to achieve a group inconsistency lower than the lowest inconsistency of all the decision makers, by aggregating their preferences. With the following proposition, we shall show that an inconsistency index satisfying the simple axiom A1 cannot be lower bounded. Perhaps due to the previous lack of an axiomatic framework which allowed the derivation of general results, to our best knowledge, this result has never been explicitly spelled out.

\begin{proposition}
\label{prop:1}
If an inconsistency index satisfies A1, then it cannot be lower bounded.
\end{proposition}

\begin{proof}
It is sufficient to consider a matrix $\mathbf{A} \notin
\mathcal{A^{*}}$ and its  transpose $\mathbf{A}^{T}$. Since they are
both inconsistent, A1 implies $I(\mathbf{A})>0$ and
$I(\mathbf{A}^{T})>0$, but the new matrix obtained as their
element-wise geometric mean $\left( \sqrt{a_{ij}a_{ji}} \right)_{n
\times n}$ is always consistent. As an inconsistency index respects
A1, then the result that $I \left( \left( \sqrt{a_{ij}a_{ji}}
\right)_{n \times n} \right) = 0 < \min \{
I(\mathbf{A}),I(\mathbf{A}^{T}) \}~\forall \mathbf{A} \notin
\mathcal{A}^{*}$ concludes the proof.
\end{proof}

Let us remark the importance of this result: for any inconsistency index satisfying the extremely weak axiom A1, it is possible that, given some pairwise comparison matrices, the inconsistency index calculated on their geometric mean-based combination be smaller than the minimum of their inconsistencies. This suggests that negotiation and convergence towards consensus might have a good effect on consistency of preferences.
Furthermore, it can happen that two decision makers, initially completely inconsistent could, by negotiating, converge towards a consensual and fully consistent solution. Consider, for instance, the case
\begin{equation}
\label{eq:two_extreme}
\mathbf{A}_{1}=
\begin{pmatrix}
1 & 1/9 & 9 \\
9 & 1   & 1/9 \\
1/9 & 9   & 1
\end{pmatrix}
\hspace{0.7cm}
\mathbf{A}_{2}=
\begin{pmatrix}
1 & 9 & 1/9 \\
1/9 & 1   & 9 \\
9 & 1/9   & 1
\end{pmatrix},
\end{equation}
whose combination with weights $\lambda_{1}=\lambda_{2}=0.5$ is the matrix
\begin{equation}
\label{eq:one_extreme}
\mathbf{A}^{\star}=
\begin{pmatrix}
1 & 1 & 1 \\
1 & 1   & 1 \\
1 & 1   & 1
\end{pmatrix} \in \mathcal{A}^{*}.
\end{equation}
This should make clear that, if the reliability of a decision should depend on the expertise of the decision makers, the evaluation of inconsistency of a consensual matrix should \emph{not} be considered as a substitute for the evaluation of single pairwise comparison matrices. Therefore, if what we want to asses by means of inconsistency tests is the ability of a decision maker, then the inconsistency of their initial preferences is more appropriate than the inconsistency of their modified/aggregated preferences. Clearly, moving towards consensus, or to an ever greater extent having his preferences aggregated with someone else's, does not make the decision maker a better and more capable expert.

To further dwell on this topic, we shall now focus on upper boundaries of inconsistency indices.

\subsection{Upper boundaries of inconsistency indices}
\label{sub:upper}
We start examining the properties of upper boundary and strong upper boundary by considering their connections with the five axioms A1--A5.

\begin{proposition}
\label{prop:2}
If an index satisfies A1 but not A3 or A4, then the index is not upper bounded.
\end{proposition}

\begin{proof}
Let us assume that index $I$ satisfies A1 but not A3. Then there exists a matrix $\mathbf{A}$ such that $I(\mathbf{A}(b')) < I(\mathbf{A}(b)) > I(\mathbf{A}(b''))$ for some $b' < b < b''$. On the other side, $\mathbf{A}(b)$ can be written as the weighted geometric mean of $\mathbf{A}(b')$ and $\mathbf{A}(b'')$, i.e.
\[
\mathbf{A}(b)= \left( a_{ij}^{b}\right) = \left( a_{ij}^{(\lambda b'+ (1-\lambda)b'')}\right) = \left( {a_{ij}^{b'}}^{\lambda} {a_{ij}^{b''}}^{1 - \lambda}\right),
\]
for a suitable $\lambda \in ]0,1[$ as in (\ref{eq:aggregation}). Then it is $I(\mathbf{A}(b)) > \max\{ I(\mathbf{A}(b')) , I(\mathbf{A}(b'')) \}$ and (\ref{eq:upper_bounded}) is violated. Therefore, $I$ is not upper bounded.\\
The proof for the case of A1 holding and A4 not holding is similar, and thus omitted.
\end{proof}

In other words, Proposition \ref{prop:2} states that axioms A3 and A4 are \emph{necessary} conditions for an index to be upper bounded. Moreover, considering that the strong upper boundary property is tighter than the upper boundary property, we can formulate the following corollary.

\begin{corollary}
\label{cor:1}
If an index satisfies A1 but not A3 or A4, then the index is not strongly upper bounded.
\end{corollary}

Building on Proposition \ref{prop:2}, and drawing on previous results showing that a number of indices do not satisfy A3 or A4 \citep[see][]{BrunelliFedrizziJORS}, we can formalize the fact that these same indices are \emph{not} upper bounded. For brevity, their definitions are omitted and the reader can refer to the original works as well as to a survey paper by \citet{BrunelliCanalFedrizziANOR}.

\begin{corollary}
\label{cor:2}
Indices $GW$ \citep{GoldenWang1989}, $HCI$ \citep{SteinMizzi2007}, $NI_{n}^{\sigma}$ \citep{RamikKorviny2010} and $RE$ \citep{Barzilai1998} are not upper bounded.
\end{corollary}

\begin{example}
 \label{Ex:1}
Consider two decision makers and the corresponding two pairwise
comparison matrices
\[
\mathbf{A}=
\begin{pmatrix}
1 & 4 & 1 \\
1/4 & 1   & 1 \\
1/1 & 1   & 1
\end{pmatrix}
\hspace{0.7cm}
\mathbf{B}=
\begin{pmatrix}
1 & 2 & 1 \\
1/2 & 1   & 7 \\
1 & 1/7   & 1
\end{pmatrix}.
\]
Their preferences are aggregated according to (\ref{eq:aggregation})
obtaining the aggregated pairwise comparison matrix $
\mathbf{A}^{\star} (\lambda) = \left( a_{ij}^{\lambda}
b_{ij}^{1-\lambda} \right)$. Then, the index $RE$ proposed by
\citet{Barzilai1998} is computed for $\lambda \in [0,1]$. In Figure
\ref{fig:Ex1} is reported the plot of this inconsistency index, i.e.
$RE\left( \mathbf{A}^{\star} (\lambda)  \right)$, as a function of
$\lambda$. It can be noted, for example, that the inconsistency of
the aggregated preferences $RE\left( \mathbf{A}^{\star} (0.5)
\right)$ is greater than the inconsistency of $ \mathbf{A}$ and
greater than the inconsistency of $ \mathbf{B}$, i.e. $RE\left(
\mathbf{A}^{\star} (0.5)  \right) > \max\{RE(\mathbf{A}),
RE(\mathbf{B})\}$. Thus, in this example, group preferences can be
more inconsistent than the preferences of all decision makers, as
index $RE$ is not upper bounded.
\begin{figure}[htbp]
    \centering
        \includegraphics[width=0.50\textwidth]{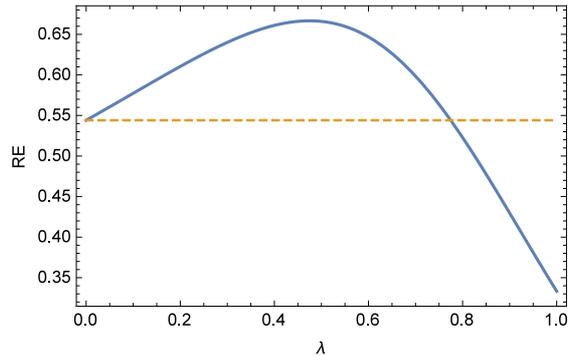}
    \caption{Example of a index which is not upper bounded: $RE$.}
    \label{fig:Ex1}
\end{figure}
\end{example}

Hereafter, until the end of the section, we shall analyze the upper boundaries of some other inconsistency indices.
Saaty's $CI$ remains the most popular inconsistency index in the literature and thus any analysis which does not take it into account, would be incomplete.
\begin{definition}[Consistency Index $CI$]
Given a pairwise comparison matrix $\mathbf{A}$ of order $n$, its Consistency Index ($CI$) is defined as:
\begin{equation}
CI(\mathbf{A})=\frac{\lambda_{\max}(\mathbf{A})-n}{n-1}.
\end{equation}
where $\lambda_{\max}(\mathbf{A})$ is the maximum eigenvalue of $\mathbf{A}$.
\end{definition}
To avoid possible confusion, we remark that the symbol $\lambda_{\max}$ has been used here, as usual,
to denote the largest eigenvalue and should not be confused with the coefficients $\lambda_{h}$.\\
 \citet{LiuEtAl2012} already proved that $CI$ is upper bounded. Here we extend their findings and show
 that it is strongly upper bounded. Let us first recall, in the following Lemma, a result by \citet{ElsnerEtAl1988}.

\begin{lemma}[\citep{ElsnerEtAl1988}]
\label{lemma:radius} Given $m$ matrices
$\mathbf{M}_{1},\ldots,\mathbf{M}_{m} \in \mathbb{R}^{n \times
n}_{+}$ and a  vector $\boldsymbol{\lambda} \in \mathcal{L}_{m}$,
then
\begin{equation}
\label{eq:radius}
\rho(\mathbf{M}^{\star}) \leq \prod_{h=1}^{m} \rho (\mathbf{M}_{h})^{\lambda_{h}}
\end{equation}
where $\rho$ denotes the spectral radius, and $\mathbf{M}^{\star}$
is the weighted  geometric mean of
$\mathbf{M}_{1},\ldots,\mathbf{M}_{m}$ as in (\ref{eq:aggregation}).
\end{lemma}

Taking into account Lemma \ref{lemma:radius}, and the inequality between arithmetic and geometric means, one derives that
\begin{equation}
\rho(\mathbf{A}^{\star}) \leq \sum_{h=1}^{m}  \lambda_{h} \rho ( \mathbf{A}_{h} ).
\end{equation}
Since the Perron-Frobenius theorem guarantees that the spectral
radius of a positive  matrix equals its maximum eigenvalue,
$\rho(\mathbf{A})=\lambda_{\max}(\mathbf{A})$, by noting that
$CI(\mathbf{A})$ is a positive affine transformation of
$\lambda_{\max}(\mathbf{A})$, the following corollary, stating the
strong upper boundedness of $CI$, can be derived.

\begin{corollary}
Index $CI$ is strongly upper bounded, i.e. the pairwise comparison matrix $\mathbf{A}^{\star}$ as in (\ref{eq:aggregation}) satisfies
\begin{equation}
\label{eq:CI_upper} CI(\mathbf{A}^{\star}) \leq  \sum_{h=1}^{m}
\lambda_{h}  CI (\mathbf{A}_{h}) ~~~~\forall
\mathbf{A}_{1},\ldots,\mathbf{A}_{m} \in
\mathcal{A},~~\boldsymbol{\lambda} \in \mathcal{L}_{m}.
\end{equation}
\end{corollary}

Among other indices, the Geometric Consistency Index ($GCI$) has probably been the most widely studied.

\begin{definition}[Geometric Consistency Index $GCI$ \citep{CrawfordWilliams1985}]
Given a pairwise comparison matrix $\mathbf{A}$ of order $n$, its Geometric Consistency Index ($GCI$) is defined as:
\begin{equation}
\label{eq:GCI}
GCI(\mathbf{A})=\frac{2}{(n-1)(n-2)}\sum_{i=1}^{n-1}\sum_{j=i+1}^{n} \, \ln^2  \left( a_{ij} \frac{\left( \prod_{k=1}^{n}a_{jk} \right)^{\frac{1}{n}}}{\left( \prod_{k=1}^{n}a_{ik} \right)^{\frac{1}{n}}}\right) .
\end{equation}
\end{definition}

\citet{EscobarEtAl2004} proved that index $GCI$ is upper bounded. Nevertheless, here we shall prove that $GCI$ is \emph{strongly} upper bounded.

\begin{proposition}
Index $GCI$ is strongly upper bounded, i.e. the pairwise comparison matrix $\mathbf{A}^{\star}$ as in (\ref{eq:aggregation}) satisfies
\begin{equation}
\label{eq:GCI_upper}
GCI(\mathbf{A}^{\star}) \leq  \sum_{h=1}^{m} \lambda_{h} GCI(\mathbf{A}_{h}) ~~~~\forall \mathbf{A}_{1},\ldots,\mathbf{A}_{m} \in \mathcal{A},~~\boldsymbol{\lambda} \in \mathcal{L}_{m}.
\end{equation}
\end{proposition}

\begin{proof}
As observed by \citet{BrunelliCritchFedrizzi2011}, $GCI(\mathbf{A})$ is proportional to the following quantity
\begin{equation}
\label{eq:prop}
\sum_{i<j<k}^{n} \left(\ln a_{ij}a_{jk}a_{ki}\right)^2.
\end{equation}
Therefore, by expanding and rearranging (\ref{eq:GCI_upper}) in the form (\ref{eq:prop}) one obtains
\begin{equation}
\label{eq:CI_proof2}
\sum_{i<j<k}^{n}
\left(
\ln \left( \prod_{h=1}^{m}{a_{ij}^{(h)}}^{\lambda_{h}}\prod_{h=1}^{m}{a_{jk}^{(h)}}^{\lambda_{h}}\prod_{h=1}^{m}{a_{ki}^{(h)}}^{\lambda_{h}}
\right) \right)^{2}
\leq
\sum_{i<j<k}^{n}
\sum_{h=1}^{m}
\lambda_{h}
 \left( \ln
\left(
a_{ij}^{(h)} a_{jk}^{(h)} a_{ki}^{(h)}\right) \right)^2 \, .
\end{equation}
If we analyze a single transitivity $(i,j,k)$, then we shall drop $\sum_{i<j<k}^{n}$ and, for notational convenience, use $x_{h}: = a_{ij}^{(h)} a_{jk}^{(h)} a_{ki}^{(h)}$. Then, we can rewrite it as
\begin{equation}
\label{eq:GCI_log}
\left(
\ln
\prod_{h=1}^{m}x_{h}^{\lambda_{h}}
\right)^{2}
\leq
\sum_{h=1}^{m} \lambda_{h}
\left(
\ln
x_{h}
\right)^{2}
\end{equation}
Considering that $\ln a + \ln b = \ln ab$ and $\ln a^{b}=b \ln a$, it becomes
\[
\left(
\sum_{h=1}^{m}
\lambda_{h}
\ln
x_{h}
\right)^{2}\leq
\sum_{h=1}^{m} \lambda_{h}
\left( \ln
x_{h} \right)^{2}.
\]
Now by putting $y_{h}:=\ln x_{h}$, we can rewrite it as
\[
\left(
\sum_{h=1}^{m}
\lambda_{h}
y_{h}
\right)^{2}
\leq
\sum_{h=1}^{m} \lambda_{h}
y_{h}^{2}.
\]
which holds thanks to the convexity of the quadratic function.
Thus, (\ref{eq:GCI_log}) is satisfied. Extending it to the sum for all the transitivity is straightforward as we know that, as it holds for all transitivities, then it must hold for their sum too.
\end{proof}

Another index, $CI^{*}$, was proposed by \citet{PelaezLamata2003} and used to improve the consistency of judgments.

\begin{definition}[$CI^{*}$ \citep{PelaezLamata2003}]
Given a pairwise comparison matrix $\mathbf{A}$ of order $n$, $CI^{*}$ is defined as:
\begin{equation}
\label{eq:det}
CI^{*}(\mathbf{A})=\sum_{i=1}^{n-2}
\sum_{j=i+1}^{n-1} \sum_{k=j+1}^{n}
 \left(\frac{a_{ij}a_{jk}}{a_{ik}} + \frac{a_{ik}}{a_{ij}a_{jk}} -2\right)\bigg/
\binom{n}{3}.
\end{equation}
\end{definition}

The validity of such an index has been corroborated by the fact that
\citet{ShiraishiEtAl1998} independently suggested another
inconsistency index ($c_{3}$) which was in effect proved to be
proportional to $CI^{*}$ by \citet {BrunelliCritchFedrizzi2011}.
Notably, both indices $CI^{*}$ and $c_{3}$ were also used as an
objective function in order to estimate missing elements in
incomplete pairwise comparison matrices \citep{ShiraishiEtAl1999}.
With the  following proposition we prove that $CI^{*}$, and
consequently also $c_{3}$, are strongly upper bounded.

\begin{proposition}
\label{prop:CI}
Index $CI^{*}$ is strongly upper bounded, i.e. the pairwise comparison matrix $\mathbf{A}^{\star}$ as in (\ref{eq:aggregation}) satisfies
\begin{equation}
\label{eq:CI*_upper}
CI^{*}(\mathbf{A}^{\star}) \leq  \sum_{h=1}^{m} \lambda_{h} CI^{*}(\mathbf{A}_{h}) ~~~~\forall \mathbf{A}_{1},\ldots,\mathbf{A}_{m} \in \mathcal{A},~~\boldsymbol{\lambda} \in \mathcal{L}_{m}.
\end{equation}
\end{proposition}

\begin{proof}
By expanding and rearranging (\ref{eq:CI*_upper}) one obtains
\begin{equation}
\begin{split}
&\sum_{i<j<k}^{n} \left(
\prod_{h=1}^{m} {a_{ij}^{(h)}}^{\lambda_{h}} \prod_{h=1}^{m} {a_{jk}^{(h)}}^{\lambda_{h}} \prod_{h=1}^{m} {a_{ki}^{(h)}}^{\lambda_{h}} +
\frac{1}{\prod_{h=1}^{m} {a_{ij}^{(h)}}^{\lambda_{h}} \prod_{h=1}^{m} {a_{jk}^{(h)}}^{\lambda_{h}} \prod_{h=1}^{m} {a_{ki}^{(h)}}^{\lambda_{h}}}-2
\right) \leq \\
& \sum_{h=1}^{m} \sum_{i<j<k}^{n} \lambda_{h} \left( a_{ij}^{(h)}a_{jk}^{(h)}a_{ki}^{(h)} + \frac{1}{a_{ij}^{(h)}a_{jk}^{(h)}a_{ki}^{(h)}} -2  \right).
\end{split}
\end{equation}
Now, similarly to what was done in the previous proof, we consider the single transitivity $(i,j,k)$ and rename the comparisons as follows: $x_{h}: = a_{ij}^{(h)} a_{jk}^{(h)} a_{ki}^{(h)}$. Thus, we obtain
\begin{equation}
\label{eq:15}
\left(
\prod_{h=1}^{m} x_{h}^{\lambda_{h}} +
\frac{1}{\prod_{h=1}^{m} x_{h}^{\lambda_{h}}}-2
\right) \leq
 \sum_{h=1}^{m} \lambda_{h} \left( x_{h} + \frac{1}{x_{h}} -2  \right)
\end{equation}
which can be rewritten as
\begin{equation}
\label{eq:CI*_4}
\underbrace{\prod_{h=1}^{m}   x_{h} ^{\lambda_{h}}}_{\text{I}} +
\underbrace{\prod_{h=1}^{m} \left( \frac{1}{x_{h}} \right)^{\lambda_{h}}}_{\text{II}}
\leq
\underbrace{\sum_{h=1}^{m}\lambda_{h} x_{h}}_{\text{III}} +
\underbrace{\sum_{h=1}^{m}\lambda_{h} \frac{1}{x_{h}}}_{\text{IV}}.
\end{equation}
Then, it is sufficient to prove that $\text{I}\leq \text{III}$ and $\text{II} \leq \text{IV}$. Namely, the following two inequalities must hold simultaneously.
\[
\prod_{h=1}^{m} x_{h}^{\lambda_{h}} \leq \sum_{h=1}^{m}\lambda_{h} x_{h}~~~~\text{and}~~~~\prod_{h=1}^{m} \left( \frac{1}{x_{h}} \right)^{\lambda_{h}} \leq \sum_{h=1}^{m}\lambda_{h} \frac{1}{x_{h}} \, .
\]
The fact that they hold derives from the inequality between the geometric and the arithmetic mean. At this stage it is proved that inequality (\ref{eq:15}) is true for single transitivities, but its extension for the sum $\sum_{i<j<k}^{n}$ is straightforward.
\end{proof}

More recently, another index, here denoted by $I_{CD}$, was proposed and grounded in the theory of Abelian linearly ordered groups \citep{CavalloD'Apuzzo2009,CavalloD'Apuzzo2010}. This index can be equivalently formulated for different types of preference relations where judgments are expressed on different scales. 

\begin{definition}[$I_{CD}$ \citep{CavalloD'Apuzzo2009}]
Given a pairwise comparison matrix $\mathbf{A}$ of order $n$, $I_{CD}$ is defined as:
\begin{equation}
\label{eq:I_CD}
I_{CD}(\mathbf{A})=\prod_{i=1}^{n-2} \prod_{j=i+1}^{n-1} \prod_{k=j+1}^{n}
\left( \max \left\{ \frac{a_{ij}a_{jk}}{a_{ik}} , \frac{a_{ik}}{a_{ij}a_{jk}}  \right\} \right)^{\frac{1}{{n\choose 3}}}.
\end{equation}
\end{definition}

With the following proposition, we state that also index $ I_{CD} $ is strongly upper bounded.

\begin{proposition}
Index $I_{CD}$ is strongly upper bounded, i.e. the pairwise comparison matrix $\mathbf{A}^{\star}$ as in (\ref{eq:aggregation}) satisfies
\begin{equation}
\label{eq:ICD_upper}
I_{CD}(\mathbf{A}^{\star}) \leq  \sum_{h=1}^{m} \lambda_{h} I_{CD} (\mathbf{A}_{h}) ~~~~\forall \mathbf{A}_{1},\ldots,\mathbf{A}_{m} \in \mathcal{A},~~\boldsymbol{\lambda} \in \mathcal{L}_{m}.
\end{equation}
\end{proposition}

\begin{proof}
If we neglect the exponent, then we can expand and rearrange (\ref{eq:ICD_upper}) as follows:
\begin{equation}
\begin{split}
& \prod_{i<j<k}^{n} \max \left\{ \prod_{h=1}^{m} {a_{ij}^{(h)}}^{\lambda_{h}} \prod_{h=1}^{m}{a_{jk}^{(h)}}^{\lambda_{h}} \prod_{h=1}^{m}{a_{ki}^{(h)}}^{\lambda_{h}} , \frac{1}{\prod_{h=1}^{m}{a_{ij}^{(h)}}^{\lambda_{h}} \prod_{h=1}^{m}{a_{jk}^{(h)}}^{\lambda_{h}} \prod_{h=1}^{m}{a_{ki}^{(h)}}^{\lambda_{h}}} \right\}
\leq\\
& \sum_{h=1}^{m} \lambda_{h} \prod_{i<j<k}^{n} \max \left\{ a_{ij}^{(h)}a_{jk}^{(h)}a_{ki}^{(h)} , \frac{1}{a_{ij}^{(h)}a_{jk}^{(h)}a_{ki}^{(h)}} \right\}.
\end{split}
\end{equation}
Now, by dropping $\prod_{i<j<k}^{n}$ we consider the single transitivity $(i,j,k)$. Moreover, by means of $x_{h}:=a_{ij}^{(h)} a_{jk}^{(h)} a_{ki}^{(h)}$ we briefly obtain
\begin{equation}
\label{eq:K_0}
\max \left\{ \prod_{h=1}^{m} x_{h}^{\lambda_{h}} , \prod_{h=1}^{m} \left( \frac{1}{x_{h}}\right)^{\lambda_{h}} \right\} \leq \sum_{h=1}^{m} \lambda_{h} \max \left\{ x_{h} , \frac{1}{x_{h}} \right\}.
\end{equation}
Now, in order to prove the previous inequality, we conjecture that there exists a quantity as follows
\begin{equation}
\label{eq:K_1}
\max \left\{ \sum_{h=1}^{m} \lambda_{h} x_{h} , \sum_{h=1}^{m} \lambda_{h} \frac{1}{x_{h}} \right\}
\end{equation}
 whose value is always bounded by the left and right hand sides of (\ref{eq:K_0}), i.e.
 \[
\max \left\{ \prod_{h=1}^{m} x_{h}^{\lambda_{h}} , \prod_{h=1}^{m} \left(  \frac{1}{x_{h}}\right)^{\lambda_{h}} \right\} \leq (\ref{eq:K_1}) \leq \sum_{h=1}^{m} \lambda_{h} \max \left\{ x_{h} , \frac{1}{x_{h}} \right\}.
\]
Hence, it is sufficient to prove the following two inequalities,
\begin{align}
\label{1}
\max \left\{ \prod_{h=1}^{m} x_{h}^{\lambda_{h}} , \prod_{h=1}^{m} \left( \frac{1}{x_{h}}\right)^{\lambda_{h}} \right\}
\leq &
\max \left\{ \sum_{h=1}^{m} \lambda_{h} x_{h} , \sum_{h=1}^{m} \lambda_{h} \frac{1}{x_{h}} \right\} \, , \\
\label{2}
\max \left\{ \sum_{h=1}^{m} \lambda_{h} x_{h} , \sum_{h=1}^{m} \lambda_{h} \frac{1}{x_{h}} \right\}
\leq &
\sum_{h=1}^{m} \lambda_{h} \max \left\{ x_{h} , \frac{1}{x_{h}} \right\} \, .
\end{align}
Inequality (\ref{1}) holds thanks to the inequality between arithmetic and geometric mean and arguments similar to those used in the proof of (\ref{eq:CI*_4}). Conversely, (\ref{2}) holds since it can be verified that both the arguments of the $\max$ operator in the left hand side of (\ref{2}) cannot be greater than the right hand side of the same inequality.
\end{proof}

Another index was introduced by \citet{Koczkodaj1993} and later studied \citep{DuszakKoczkodaj1994} and compared with other indices, for example with $CI$ \citep{Bozoki}.

\begin{definition}[Index $K$ \citep{DuszakKoczkodaj1994}]
Given a pairwise comparison matrix $\mathbf{A}$ or order $n$, the inconsistency index $K$ is defined as:
\begin{equation}
\label{eq:K}
K(\mathbf{A})=
\max_{i < j < k} \left\{
\min \left\{ \bigg| 1-a_{ij}a_{jk}a_{ki}\bigg| , \bigg| 1-\frac{1}{a_{ij}a_{jk}a_{ki}} \bigg| \right\} \right\}.
\end{equation}
\end{definition}

\begin{proposition}
The inconsistency index $K$ is upper bounded, i.e. the pairwise comparison matrix $\mathbf{A}^{\star}$ as in (\ref{eq:aggregation}) satisfies
\[
K(\mathbf{A}^{\star}) \leq  \max \{ K(\mathbf{A}_{1}),\ldots, K(\mathbf{A}_{m}) \}  ~~~~\forall \mathbf{A}_{1},\ldots,\mathbf{A}_{m} \in \mathcal{A},~~\boldsymbol{\lambda} \in \mathcal{L}_{m}.
\]
\end{proposition}

\begin{proof}
Here again we consider the case with only one transitivity $(i,j,k)$. Then we can start by writing
\begin{equation}
\label{eq:proofK}
\begin{split}
&\min \left\{ \bigg| 1- \prod_{h=1}^{m} {a_{ij}^{(h)}}^{\lambda_{h}}  \prod_{h=1}^{m}{a_{jk}^{(h)}}^{\lambda_{h}}  \prod_{h=1}^{m}{a_{ki}^{(h)}}^{\lambda_{h}} \bigg|,
\bigg| 1- \frac{1}{\prod_{h=1}^{m} {a_{ij}^{(h)}}^{\lambda_{h}}  \prod_{h=1}^{m}{a_{jk}^{(h)}}^{\lambda_{h}}  \prod_{h=1}^{m}{a_{ki}^{(h)}}^{\lambda_{h}}} \bigg|
\right\}
\leq \\
&\max_{h} \min \left\{ \bigg|1-a_{ij}^{(h)} a_{jk}^{(h)} a_{ki}^{(h)}\bigg|, \bigg| 1-\frac{1}{a_{ij}^{(h)} a_{jk}^{(h)} a_{ki}^{(h)}} \bigg| \right\},
\end{split}
\end{equation}
now with $x_{h}:=a_{ij}^{(h)}a_{jk}^{(h)}a_{ki}^{(h)}$, we have
\begin{equation}
\min \left\{ \bigg| 1- \prod_{h=1}^{m} x_{h}^{\lambda_{h}} \bigg|,
\bigg| 1-   \frac{1}{\prod_{h=1}^{m} x_{h}^{\lambda_{h}}}  \bigg| \right\}
\leq
\max_{h} \min \left\{ \bigg|1- x_{h} \bigg|, \bigg| 1-\frac{1}{x_{h}} \bigg| \right\}.
\end{equation}
By using the following notational change, $f(x):=\min \left\{ \big|1- x \big|, \big| 1-\frac{1}{x} \big| \right\}$, it is possible to rewrite the previous inequality as
\begin{equation}
f \left( \prod_{h=1}^{m} x_{h}^{\lambda_{h}}\right)
\leq
\max  \left\{ f(x_{1}),\ldots,f(x_{m}) \right\}.
\end{equation}
Now we shall prove it by contradiction. Let us suppose that the proposition is wrong and that there exist $x_{1},\ldots,x_{m}$ and $\lambda_{1},\ldots,\lambda_{m}$ such that
\[
f \left( \prod_{h=1}^{m} x_{h}^{\lambda_{h}}\right)
>
\max  \left\{ f(x_{1}),\ldots,f(x_{m}) \right\}.
\]
Then, by examining the function $f(x)$ one can see that it is decreasing for $x\in ]0,1[$ and increasing for $x\in ]1,\infty[$ with a minimum in $1$ and then deduce that either $\prod_{h=1}^{m} {x_{h}}^{\lambda_{h}}>\max\{ x_{1},\ldots,x_{m} \}$ or $\prod_{h=1}^{m} {x_{h}}^{\lambda_{h}}<\min\{ x_{1},\ldots,x_{m} \}$. However, this cannot be possible, as we know that the weighted geometric mean is an averaging aggregation function, i.e. $\min \{ x_{1},\ldots,x_{m} \} \leq \prod_{h=1}^{m}x_{h}^{\lambda_{h}} \leq \max \{ x_{1},\ldots,x_{m} \}$. Having proved (\ref{eq:proofK}) for each triplet $(i, j, k)$, the original inequality in the proposition holds in particular for the maximum, since the maximum of quasiconvex functions is quasiconvex.
\end{proof}

\begin{remark}
Although index $K$ is upper bounded, it is  not strongly upper
bounded.  Namely, given a pairwise comparison matrix
$\mathbf{A}^{\star}$ as in (\ref{eq:aggregation}) the following
property does \emph{not} hold
\begin{equation}
\label{eq:K_upper_weak} K(\mathbf{A}^{\star}) \leq  \sum_{h=1}^{m}
\lambda_{h}  K (\mathbf{A}_{h}) ~~~~\forall
\mathbf{A}_{1},\ldots,\mathbf{A}_{m} \in
\mathcal{A},~~\boldsymbol{\lambda} \in \mathcal{L}_{m} .
\end{equation}
In order to prove that (\ref{eq:K_upper_weak}) does not hold in
general, let us present the following counterexample. Consider two
decision makers and the corresponding two pairwise comparison
matrices
\[
\mathbf{A}_{1}=
\begin{pmatrix}
1 & 1/9 & 9 \\
9 & 1   & 1/9 \\
1/9 & 9   & 1
\end{pmatrix}
\hspace{0.7cm}
\mathbf{A}_{2}=
\begin{pmatrix}
1 & 3 & 9 \\
1/3 & 1   & 3 \\
1/9 & 1/3   & 1
\end{pmatrix}.
\]

\noindent As in Example \ref{Ex:1}, the preferences are aggregated
according to (\ref{eq:aggregation}) obtaining matrix $
\mathbf{A}^{\star} (\lambda) $. Then, the index
$K(\mathbf{A}^{\star} (\lambda))$
 is computed for $\lambda \in [0,1]$.
As illustrated in Figure \ref{fig:Ex2}, $K$ is not strongly upper
bounded. For instance, taking $\lambda = 0.5$ ,

\[
K(\mathbf{A}^{\star}(0.5)) \approx 0.96 > 0.5 K(\mathbf{A}_{1})+ 0.5
K(\mathbf{A}_{2}) \approx 0.5.
\]

\noindent Thus, in this example, the inconsistency of the aggregated
preferences is greater than the average inconsistency of the
decision makers, as index $K$ is not strongly upper bounded.

\begin{figure}[htbp]
    \centering
        \includegraphics[width=0.50\textwidth]{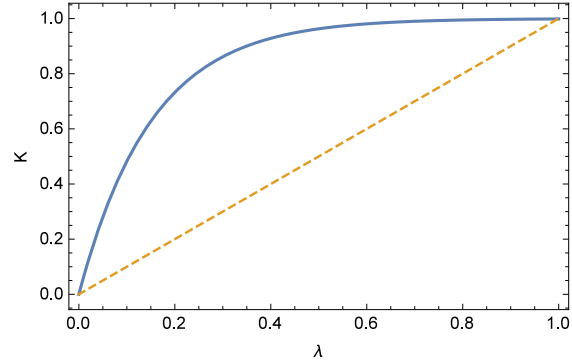}
    \caption{Index $K$ is upper bounded but not strongly upper bounded.}
    \label{fig:Ex2}
\end{figure}

\end{remark}

\section{Discussion}
\label{sec:discussion}

In the previous section we analyzed some boundary properties and whether they are respected or not by some well-known inconsistency indices. To do so, we also related these new properties with the axioms A1--A5 introduced by \citet{BrunelliFedrizziJORS} and discussed some implications. One result is that the inconsistency of a combination of pairwise comparison matrices cannot be lower bounded by the value of inconsistency of the least inconsistent matrix. We have in fact shown that an inconsistency index satisfying A1 is never lower bounded with respect to the aggregation of the preferences of different decision makers. Moreover, results show that, more often than not, inconsistency indices are, instead, upper bounded. Table \ref{tab:summary} summarizes the results obtained in the previous section.

\begin{table}[hbt]
\begin{center}
\begin{tabular}{c|c|c|c}
Index / Property     & ~~LB~~ & ~~UB~~ & ~~S-UB~~   \\
    \hline
$CI$                        & \xmark   & \underline{\cmark}   & \cmark    \\
$RE$                        & \xmark   & \xmark   & \xmark    \\
$CI^{*}$, $c_{3}$ & \xmark   & \cmark   & \cmark    \\
$GCI$           & \xmark   & \underline{\cmark}   & \cmark    \\
$HCI$                       & \xmark   & \xmark   & \xmark    \\
$GW$                        & \xmark   & \xmark   & \xmark    \\
$I_{CD}$                        & \xmark   & \cmark   & \cmark    \\
$K$                         & \xmark   & \cmark   & \xmark    \\
$NI_{n}^{\sigma}$ & \xmark   & \xmark   & \xmark    \\
    \hline
\end{tabular}
\caption{Summary table. LB = lower bounded, UB = upper bounded, S-UB = strongly upper bounded, \cmark = the property is satisfied, \xmark = the property is not satisfied. The underlined results were already known in the literature.} \label{tab:summary}
\end{center}
\end{table}

Comparing the results summarized in Table \ref{tab:summary} with the satisfaction of the axioms A1--A5, it appears that all the indices which satisfy the axioms A1--A5 are also upper bounded. From this, it is only natural to hypothesize that the axioms A1--A5 mathematically imply the satisfaction of the upper boundary condition. However, this is not true and can be formalized in the following proposition.

\begin{proposition}
\label{prop:A5UB}
If an inconsistency index satisfies the axiomatic system A1--A5, then it is not necessarily upper bounded.
\end{proposition}
\noindent To prove the proposition, it is sufficient to show a counterexample; that is, an index satisfying A1--A5 which is not upper bounded. The following index is an example:
\begin{equation}
\label{eq:I_Mikko}
I_{M}(\mathbf{A})=\min_{i<j<k}\left| \ln \frac{a_{ij}a_{jk}}{a_{ik}}\right| + \sum_{i=1}^{n-2} \sum_{j=i+1}^{n-1} \sum_{k=j+1}^{n} \left| \ln \frac{a_{ij}a_{jk}}{a_{ik}}\right| .
\end{equation}
It is easy to check that $I_{M}$ satisfies the five axioms but, when applied to the following two matrices
\[
\mathbf{A}=
\begin{pmatrix}
1    & 4 & 1 & 1 \\
1/4  & 1 & 1 & 1 \\
1  & 1 & 1 & 1 \\
1  & 1 & 1 & 1 \\
\end{pmatrix},
\hspace{0.7cm}
\mathbf{B}=
\begin{pmatrix}
1    & 1 & 1 & 1 \\
1  & 1 & 1 & 1 \\
1  & 1 & 1 & 4 \\
1  & 1 & 1/4 & 1 \\
\end{pmatrix},
\]
we obtain $I_{M}(\mathbf{A})=I_{M}(\mathbf{B}) \approx 2.77$, but if
we  consider that $I_{M}(a_{ij}^{0.5} b_{ij}^{0.5})\approx 3.46574$,
it appears that the proposition is correct. A graphical
representation of this case is presented in Figure \ref{fig:Ex3}.
\begin{figure}[htbp]
    \centering
        \includegraphics[width=0.50\textwidth]{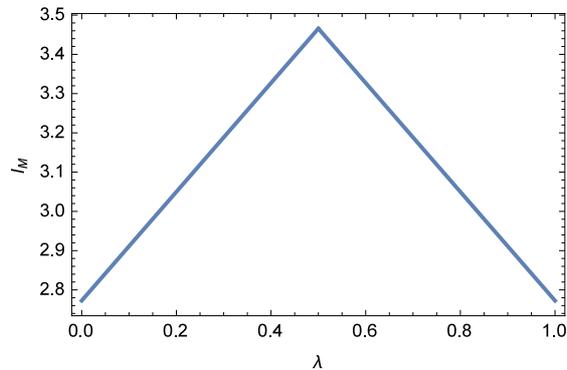}
    \caption{Index $I_{M}$ satisfies A1--A5 but is not upper bounded.}
    \label{fig:Ex3}
\end{figure}

Another matter, this time of a more qualitative debate, relates with
the meaning  attached to the consistency of group preferences. If
inconsistency indices are used to estimate the rationality of
decision makers and detect those which are too irrational, then the
meaning of the inconsistency of the aggregated preferences vanishes.
This is due to the fact that the link between rationality of a
decision maker and consistency of preferences, which is generally
assumed in the practice of the AHP, holds only for the initial
preferences of a decision maker and breaks down as the preferences
are aggregated, as already discussed in Subsection \ref{sub:lower}.
We believe that we can learn a lot from the boundary properties: for
instance we know, a priori, that when evaluated by an upper bounded
inconsistency index, if the pairwise preferences of the most
inconsistent decision maker move towards a consensual solution, at
least at the beginning, they become more consistent. Nevertheless,
we refrain---and we advise other researchers to do the same---from
considering the inconsistency of the group preferences as a global
measure of the inconsistency of various decision makers.


\section{Conclusions}
\label{sec:conclusions}

From our investigation on the connection between consensus and
inconsistency, we can summarize some relevant findings. First, the
effect of the preference aggregation among different decision makers
on the inconsistency of the group preferences depends crucially on
the inconsistency index that is used. More precisely, we proved that
a certain number of already known indices satisfy the upper boundary
property and/or the strong upper boundary property. As a
consequence, for some indices the inconsistency of the aggregated
preferences is always lower than a weighted mean of the
inconsistencies of the original preference relations. This effect
can be synthesized by saying that preference aggregation is
consistency improving. On the other hand, if the inconsistency is
evaluated by means of other indices, then the opposite result can be
obtained. Therefore, as pointed out in \citet{BrunelliFedrizziJORS},
a suitable choice of an inconsistency index is a crucial phase in
decision-making processes, since the use of different methods for
measuring consistency can lead to different conclusions and can
affect the decision outcome in practical applications.
Interestingly, some more general results have been derived from the
axiomatic system proposed by \citet{BrunelliFedrizziJORS}. If it is
true that one of the merits of an axiomatic system is its fertility,
i.e. the capacity to produce propositions, then Propositions
\ref{prop:1} and \ref{prop:2} and Corollaries \ref{cor:1} and
\ref{cor:2} are positive signs in this direction and seem to
indicate that the axiomatic system can be used to derive interesting
results. Hence, we believe that future research on inconsistency of
pairwise comparisons (i) can build upon the axiomatic system A1--A5
and (ii) further investigate the relation between consensus and
consistency, perhaps clarifying the open issues exposed in the
previous section.

\section*{Acknowledgements}
We thank the three anonymous reviewers for their constructive
comments, and S\'{a}ndor B\'{o}zoki for having read a draft of this
manuscript. We are also grateful to Mikko Harju from the Systems
Analysis Laboratory, Aalto University, who provided the example of
function $I_{M}$ in (\ref{eq:I_Mikko}). The research of the first
author is supported by the Academy of Finland.

\end{document}